\documentclass[acmsmall,authorversion]{acmart}

\AtBeginDocument{%
	\providecommand\BibTeX{{%
			\normalfont B\kern-0.5em{\scshape i\kern-0.25em b}\kern-0.8em\TeX}}}

\usepackage{epsfig}
\usepackage{algorithm}
\usepackage{algpseudocode}  
\usepackage{subfigure} 
\usepackage{multirow}
\usepackage{booktabs}
\usepackage{color}
\usepackage{enumerate}
\usepackage{url}
\newtheorem{lemma}{Lemma}
\newtheorem{theorem}{Theorem}
\algdef{SE}[DOWHILE]{Do}{doWhile}{\algorithmicdo}[1]{\algorithmicwhile\ #1}%

\begin{document}
\title{Efficient Algorithms for Rotation Averaging Problems}
\author{Yihong Dong}

\email{1931543@tongji.edu.cn}
\author{Lunchen Xie}
\email{lcxie@tongji.edu.cn}
\author{Qingjiang Shi}
\authornote{Corresponding Author: Qingjiang Shi}
\email{shiqj@tongji.edu.cn}
\affiliation{%
	\institution{School of Software Engineering, Tongji University}
	\city{Shanghai}
	\country{China}
}

\setcopyright{acmcopyright}
\acmJournal{PACMCGIT}
\acmYear{2021} \acmVolume{4} \acmNumber{1} \acmArticle{} \acmMonth{5} \acmPrice{15.00}\acmDOI{10.1145/3451263}

\begin{abstract}
	The rotation averaging problem is a fundamental task in computer vision applications. It is generally very difficult to solve due to the nonconvex rotation constraints. While a sufficient optimality condition is available in the literature, there is a lack of a fast convergent algorithm to achieve stationary points. In this paper, by exploring the problem structure, we first propose a block coordinate descent (BCD)-based rotation averaging algorithm with guaranteed convergence to stationary points. Afterwards, we further propose an alternative rotation averaging algorithm by applying successive upper-bound minimization (SUM) method. The SUM-based rotation averaging algorithm can be implemented in parallel and thus is more suitable for addressing large-scale rotation averaging problems. Numerical examples verify that the proposed rotation averaging algorithms have superior convergence performance as compared to the state-of-the-art algorithm. Moreover, by checking the sufficient optimality condition, we find from extensive numerical experiments that the proposed two algorithms can achieve globally optimal solutions. 
\end{abstract}

\ccsdesc[500]{Computing methodologies~Image manipulation}

\keywords{rotation averaging, reconstruction, BCD, SUM}

\maketitle

\section{Introduction}
The rotation averaging problem has been intensively studied through these years. The objective of the problem is to determine the absolute camera orientations given  a bunch of relative rotation estimates between pairs of poses. The problem has vast applications in computer vision, robotics, sensor networks and related areas. For example, in computer vision, this problem is used to produce camera orientations and camera locations from a set of images about one scene \cite{dai2009rotation}. By extracting and purifying such information, a 3D reconstruction can be then established, which is called the structure from motion problem \cite{koenderink1991affine}. Motivated by the vast number of applications, developing effective and efficient algorithms for the rotation averaging problem is of great importance in practice.

The mathematical problem of the rotation averaging problem is computing the absolute rotation matrices given the relative rotation matrices between cameras. This can be thought of as inferring a graph of cameras. Each camera is represented as a vertex of the graph, while each undirected edge $(i, j)$ represents a relative rotation measurement $R_{ij}$. The goal is to find absolute rotation matrix $R_i$ for every vertex $i$ such that $R_iR_{ij}$ = $R_j$ holds. In practice, due to the noise in measurements of relative rotations, one can only achieve $R_iR_{ij} \approx R_j$, and an optimization problem is usually solved to compute the absolute rotation matrices. The problem is considered to be troublesome because the rotation constraints on the rotation matrix are highly nonconvex \cite{wilson2016rotations}.

\begin{figure*}[htb]
	\setlength{\abovecaptionskip}{0.cm}
	\setlength{\belowcaptionskip}{-0.5cm}   
	\begin{center}
		\includegraphics[width=14cm]{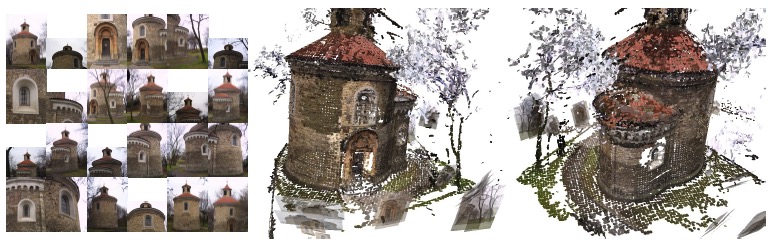}
	\end{center}
	\footnotesize{\qquad \quad (a) 24 out of 124 images \qquad (b) Front view of the dense reconstruction \quad (c) Back view of the dense reconstruction}
	\caption{Dense reconstruction of St. Martin rotunda with 124 images taken from different views \cite{martinec2007robust}. \copyright Daniel Martinec}
	\label{dense reconstruction}
\end{figure*}

In this paper, we present two fast and accurate approaches for computing the absolute rotation matrices. By defining the relative rotation measurement matrix $R_{ij}$, we first use Block Coordinate Descent method (BCD) to calculate the $R_{i}$ iteratively \cite{wright2015coordinate}. Since BCD can only be executed serially on different coordinates (different rotation matrices in our case), to further speed up the computation of absolute rotation matrices, we develop a Successive Upper-bound Minimization (SUM) based algorithm \cite{razaviyayn2013unified}. The SUM-based algorithm can run in parallel on $R_i$'s and the possible parallelization is achieved by a novel transformation of the original problem that decouples the objective functions and constraints on different $R_i$'s. Both of these algorithms are guaranteed to converge to stationary points and by numerically checking the global optimality condition in \cite{eriksson2018rotation}, we find that the algorithms also converge to global optimum in all our numerical examples. 

In summary, the main contribution of this paper is that we present a BCD-based and a SUM-based algorithm for the rotation averaging problem. The proposed algorithms significantly outperform the state-of-the-art method and are shown to find solutions satisfying the sufficient global optimality condition. 

\section{Related Work}
In 3D reconstruction, the common approach is to estimate the relative poses of cameras from pairs of images first, then estimate their absolute poses from the relative poses.  The task of estimating relative poses is usually formulated as optimization problems. The approaches for estimating the relative poses can be categorized into sequential methods or non-sequential methods, depending on the order of adding the camera points. Sequential methods may be severely affected by the choice of initial pair of images. In addition, the iteration process of adding new cameras is not stable as the order of camera will affect the final structure. The well-known BUNDLER \cite{snavely2006photo} is a typical and mature system of sequential methods. Though traditional methods will be severely affected by the short baseline problem (i.e., the estimation error of relative poses can be large when two cameras are too close to each other), it has been proven that even though short baseline will result in large error of translation, the rotation will still be well estimated. Thus, the reconstruction can be performed without concern for initial choice of image pairs. Non-sequential methods like \cite{enqvist2011non} can get rid of the limitations mentioned above. For example, the approach in \cite{enqvist2011non} perform cycle removal, then the graph edges are weighted and a maximum spanning tree (MST) can be extracted. We can estimate rotations via MST. But they may introduce other weaknesses, the method is highly dependent on the chosen MST, so only the correct tree leads to correct rotation estimation. The work \cite{rasmuson2020user} presents a user-guided system for accessible 3D reconstruction and modeling of real-world objects using multi-view stereo. 

For calculating absolute poses, before the thorough study on rotation matrix based methods, quaternions are usually used for such problems. By presenting rotation averaging problem with quaternions, relaxation about constraints on rotation parameterizations was introduced, the whole problem can be viewed as a linear homogeneous least squares problem \cite{govindu2001combining}. However such quaternions should be kept in length one, which introduces additional norm constraint. In some cases the norm may be far from one, thus resulting in poor estimation \cite{martinec2007robust}. Moreover, in order to obtain the result fast and directly, they solved the problem approximately by omitting orthogonality and determinant constraints (in other words, without enforcing rotational constraints) on the 3 by 3 rotation matrix.

By introducing rotation matrices, convexity properties of the single rotation averaging problem are given in \cite{hartley2010rotation}, making a generalization of the concept of geodesic convexity. This leads to nice characterizations of the convexity properties of cost functions used in rotation averaging. However, the results do not generalize to the case of multiple rotations, in other words, our research interest. Fortunately, it is covered in \cite{hartley2013rotation}. Thus, rotation averaging problem became more complicated due to the demand of enforcing non-convex rotation constraints. To overcome these obstacles, researchers took advantages of Lagrangian Duality \cite{fredriksson2012simultaneous,eriksson2018rotation}, which can turn non-convex problems into convex forms and handle constraints of rotation matrices. However, these approaches derived the solution in dual space or relaxed space. We aim to exploit the properties of rotation matrices to provide a direct solution.

It is noteworthy that most of the existing methods are serial algorithms that cannot exploit the benefit of parallel computation. In contrast to these serial algorithms, our algorithms include one that can take advantage of parallel computing, which can be significantly faster when the number of cameras is large. 

\section{Problem Definition}
This section formally presents the rotation averaging problem and its mathematical formulation. To be clear, let us start with the definition of rotation. The set of rotations about the origin in three dimensional Euclidean space is referred to as the Special Orthogonal Group, denoted as $SO(3)$ given by
\begin{equation}
SO(3)\triangleq\left\{R\in\mathbb{R}^{3\times3} \big| R^TR=I, det(R)=1 \right\}
\end{equation}
where $I$ is the three dimensional identity matrix and $det()$ is the matrix determinant. Using the above definition, the rotation averaging problem can be simply stated as follows. \emph{Given the estimated relative rotations $R_{ij}\in \mathbb{R}^{3\times3}$, $i,j=1,2,\ldots, n$, we want to recover $n$ absolute rotations $R_i\in SO(3)$, $i=1,2,\ldots n$, from the pairwise relation $R_iR_{ij}\approx R_j$} for some $i,j$. 

In the typical application of rotation averaging, multiview construction \cite{martinec2007robust}, the relative rotation is computed by massive images taken from different views followed by calculation of the absolute rotations and the reconstruction of the global view of some objects. An example is shown in Fig. \ref{dense reconstruction}.    




Generally, the relative rotations are obtained with estimation errors. To take the estimation errors into consideration, the rotation averaging problem is usually formulated as
\begin{equation}
\mathop {\min }\limits_{R_1,...,R_n\in SO(3)} \sum_{(i,j)\in E} dist(R_iR_{ij},R_j)^2\label{minimization problem}
\end{equation}
where $E$ denotes the set of pair $(i,j)$ for which the relative rotation is estimated, and $dist(A,B)$ represents a measure of distance between two rotations on $SO(3)$. Typical distance measures include angular distance (a.k.a. geodesic distance),  chordal distance, and quaternion distance, etc. \cite{andoni2015practical,sakai1994labeling,hanson2005visualizing}. 

In this paper, we restrict our attention to the commonly used chordal distance (equivalently in a least-square sense). That is, $dist(\cdot,\cdot)$ in \eqref{minimization problem} is taken to be the Frobenius norm, leading to 
\begin{equation}\label{eq:fnorm}
\mathop {\min }\limits_{R_1,...,R_n\in SO(3)} \sum_{(i,j)\in E} \big|\big|R_iR_{ij}-R_j\big|\big|^2_F.
\end{equation}

While problem \eqref{eq:fnorm} has simple quadratic objective function,  it is nonconvex and  extremely difficult to solve due to the difficult nonconvex constraints $SO(3)$. The work \cite{eriksson2018rotation} provided a \emph{sufficient optimality condition} for problem \eqref{eq:fnorm}, which we summarized in Theorem 1 with the definition
\begin{equation}  \label{Rtilde}
\begin{aligned}
&\tilde{R} \triangleq \left[
\begin{matrix}
0 & \tilde{R}_{12} & ... & \tilde{R}_{1n} \\
\tilde{R}_{21} & 0 & ... & \tilde{R}_{2n}\\
\vdots & &\ddots & \vdots\\
\tilde{R}_{n1} & \tilde{R}_{n2} & ... &\tilde{R}_{nn}
\end{matrix}
\right]  
\end{aligned}
\end{equation}
where $\tilde{R}_{ij} = R_{ij}$ if $(i,j)\in E$, otherwise $\tilde{R}_{ij}=0$.
\begin{theorem}[\cite{eriksson2018rotation}] \label{lemma6}
	A stationary point $\{R_i^*\}$ of the rotation averaging problem is also an optimum solution if it satisfies $\Lambda-\tilde{R}\succeq 0$ where 
	\begin{equation}  
	\begin{aligned}
	&\Lambda \triangleq \left[
	\begin{matrix}
	\Lambda_1 & 0 & ... & 0 \\
	0 & \Lambda_2 & ... & 0\\
	\vdots & &\ddots & \vdots\\
	0 & 0 & ... & \Lambda_n
	\end{matrix}
	\right]
	\end{aligned}
	\end{equation}
	with
	$\Lambda_i \triangleq \sum_{j\in \{j~|~ j\neq i, (i,j)\in E\}} \tilde{R}_{ij}R^{*T}_jR^*_i$. 
\end{theorem}

\emph{\textbf{Remark 1}:
	Theorem 1 was obtained via strong duality theory by exploring the connection between the rotation averaging problem and its dual problem (or the dual of the dual problem, named DD problem for shorthand). It is useful for optimality verification of a given stationary point. To obtain a possibly stationary point, \cite{eriksson2018rotation} further developed a BCD-type algorithm for the DD problem in the form of} 
\begin{equation}
\begin{split}
& \min_{Y}~~ -tr(\tilde{R}Y)\\
& s.t.~Y\succeq 0, Y_{ii}=I, \forall i.
\end{split}\label{RY}
\end{equation}
with 
\begin{equation}  
\begin{aligned}
&Y \triangleq \left[
\begin{matrix}
Y_{11} & Y_{12} & ... & Y_{1n} \\
Y_{21} & Y_{22} & ... & Y_{2n}\\
\vdots & &\ddots & \vdots\\
Y_{n1} & Y_{n2} & ... &Y_{nn}
\end{matrix}
\right]  
\end{aligned}
\end{equation}
\emph{where each block $Y_{ij}$ is a $3$-by-$3$ matrix. Then a primal solution (i.e., rotation matrices) is recovered from the first block row of $Y$. However, since the constraint $Y\succeq0$ couples all block variables $Y_{ij}'s$, the convergence of their algorithm could be very slow as shown later in the simulations.} 

In what follows, differently from the BCD algorithm in \cite{eriksson2018rotation} developed for the dual of dual problem of the rotation averaging problem, we propose efficient algorithms directly for the rotation averaging problem \eqref{eq:fnorm}. The proposed algorithms can achieve stationary points of the rotation averaging problem. Given the stationary points, we invoke Theorem 1 to verify whether it is an optimal solution.

\section{Methodology}
In this section, based on block coordinate descent method and successive upper-bound minimization (SUM) method, we propose two efficient rotation averaging algorithms. The BCD-based algorithm runs in serial while the SUM-based algorithm can be implemented in parallel. Moreover, both algorithms have guaranteed convergence to stationary points of the rotation averaging problem.  

\subsection{Block Coordinate Descent}
BCD is an optimization method that successively minimizes along block coordinate directions to find the minimum of a function. Specifically, at each iteration, BCD method optimizes only one of the block variables while the rest of blocks are kept fixed. It generally works extremely efficiently once the single-block subproblem can be solved easily. It is observed that the $SO(3)$ constraints of problem \eqref{eq:fnorm} are separable across the rotation matrices. Hence, the BCD method naturally applies to the rotation averaging problem but with the single SO(3) constrained subproblem remained to be addressed.

To develop an efficient BCD-based rotation averaging algorithm, we further simplify the rotation averaging problem \eqref{eq:fnorm}. Since matrices $R_i$, $R_{ij}$ and $R_j$ are all orthonormal, both $\Vert R_iR_{ij}\Vert^2_F$ and $\Vert R_j\Vert^2_F$ are equal to $3$. Consequently, problem \eqref{eq:fnorm} can be equivalently simplified as follows
\begin{equation}\label{Inital problem}
\mathop {\min }\limits_{R_1,...,R_n\in SO(3)} -\sum_{(i,j)\in E} tr \left( R_iR_{ij}R^T_j  \right) .
\end{equation}
It is observed that if we only optimize the i-th rotation matrix $R_i$, the objective function is linear in $R_i$. This observation is useful in the development of BCD-based algorithms.

In the BCD method applied to \eqref{Inital problem}, while fixing other blocks, the $l$-th subproblem with respect to $R_l$  is given by
\begin{align}
	\begin{split}\label{sub-problem}
		&\min_{R_l \in \mathbb{R}^{3\times3}} \quad -\sum_{(i,j)\in E} tr \left( R_iR_{ij}R^T_j  \right)\\
		&s.t.\quad\quad R^T_lR_l = I, det(R_l)=1.
	\end{split}
\end{align} 
Define
\begin{equation}
A_l \triangleq -\sum_{(l,q)\in E} R_{lq}R^T_q  -\sum_{(p,l) \in E} R^T_{pl}R^T_p\label{A}
\end{equation}
where the two summations are taken over $p$ and $q$, respectively.
Then problem \eqref{sub-problem} can be equivalently written in a neat form as follows
\begin{align}
	\begin{split}\label{RA}
		&\min_{R_l \in \mathbb{R}^{3\times3}} \quad tr\left(A_lR_l\right) \\
		&s.t. \quad R^T_lR_l = I, det(R_l)=1.\\
	\end{split}
\end{align}
See Appendix A for detailed derivation of \eqref{RA} from \eqref{sub-problem}. In what follows, we show how the above problem can be globally solved. 

\subsubsection{Linear optimization with single SO(3) constraint}
The core of the BCD-based rotation averaging algorithm is solving problem \eqref{RA}, which we refer to as the problem of linear optimization with single $SO(3)$ constraint, named LOSSO for shorthand \cite{hartley2013rotation}.

For notational simplicity, we rewrite the LOSSO problem as follows
\begin{equation}\label{eq:LOSSO}
\begin{split}
& \min_{X\in\mathbb{R}^{3\times3}} ~tr(AX)\\
& ~~s.t.~X^TX=I, \det(X)=1.
\end{split}
\end{equation}
Note that due to the constraint $\det(X)=1$, the problem cannot be easily solved via matrix decomposition at the first glance. However, for a three dimensional matrix $X\in SO(3)$, it has a parametric expression, i.e., axis-angle representation, as follows \cite{taylor1994minimization}

\begin{align*}
	X(u,\theta)= 
	\left[
	\begin{matrix}
		\cos{\theta} +\delta u^2_1 & \delta u_1u_2-u_3\psi & \delta u_1u_3+u_2\psi\\
		\delta u_2u_1+u_3\psi & \cos{\theta} +\delta u^2_2 & \delta u_2u_3-u_1\psi \\
		\delta u_3u_1-u_2\psi & \delta u_3u_2+u_1\psi & \cos{\theta} +\delta u^2_3
	\end{matrix}
	\right]
\end{align*}
where $u\triangleq (u_1, u_2, u_3)$ satisfies $u^2_1+u^2_2+u^2_3=1$, $\delta \triangleq 1-\cos{\theta}$, and $\psi \triangleq\sin{\theta} $ with $-\pi\leq \theta \leq \pi$.
Hence, problem \eqref{eq:LOSSO} can be converted to an optimization problem with respect to $u$ and $\theta$, i.e.,
\begin{equation}\label{eq:LOSSO2}
\begin{split}
& \min_{u, \theta} ~tr(AX(u,\theta))\\
& ~~s.t.~u^2_1+u^2_2+u^2_3=1, -\pi\leq \theta\leq \pi.
\end{split}
\end{equation}
It can be seen that if $A$ is diagonal, then the problem becomes much easier. Therefore, let us first simplify the LOSSO problem \eqref{eq:LOSSO} and explore the solution structure which is stated in the following lemma.
\begin{lemma}\label{lemma2}
	Let $A=\hat{U}DV^T$ be the SVD of A. Suppose $\hat{D}$ and $\hat{V}$ fulfill the following conditions
	\begin{equation}  
	\left\{ 
	\begin{aligned}\label{eq:cond}
	&\hat{D} = -D, \hat{V} = -V, \quad if\,\, det(\hat{U})*det(V)=-1\\
	&\hat{D} = D, \hat{V} = V, \quad if\,\, det(\hat{U})*det(V)=1.
	\end{aligned}
	\right.  
	\end{equation} 
	Then $X=\hat{V}\hat{\Sigma}\hat{U}^T$ is an optimal solution to problem LOSSO, where 
	\begin{equation}\label{eq:LOSSO3}
	\begin{split}
	& \hat{\Sigma}\in\arg\min_{\Sigma\in\mathbb{R}^{3\times3}} ~tr(\hat{D}\Sigma)\\
	& ~~s.t.~\Sigma^T\Sigma=I, \det(\Sigma)=1.
	\end{split}
	\end{equation}
\end{lemma}
\begin{proof}
	It suffices to show that the LOSSO problem can be reduced to \eqref{eq:LOSSO3}. Due to $A=\hat{U}DV^T$ and \eqref{eq:cond}, we have $A=\hat{U}\hat{D}\hat{V}^T$ and $det(\hat{U})\det(\hat{V})=1$. By further noting that $\hat{U}$ and $\hat{V}$ are orthogonal matrices, we obtain 
	$$X^TX=I\Longleftrightarrow\left(\hat{V}^TX\hat{U}\right)^T\hat{V}^TX\hat{U}=I,$$ 
	$$det(X)=1\Longleftrightarrow\det(\hat{V}^TX\hat{U})=1$$
	and
	$$tr(AX)=tr(\hat{U}\hat{D}\hat{V}^TX)=tr(\hat{D}\hat{V}^TX\hat{U})$$ where the second equality is due to the identity $tr(BC)=tr(CB)$. Therefore, the LOSSO problem can be equivalently written as
	\begin{equation}\label{eq:LOSSO4}
	\begin{split}
	& \min_{X\in\mathbb{R}^{3\times3}} ~tr(\hat{D}\hat{V}^TX\hat{U})\\
	& ~~s.t.~\left(\hat{V}^TX\hat{U}\right)^T\hat{V}^TX\hat{U}=I, \det(\hat{V}^TX\hat{U})=1.
	\end{split}
	\end{equation}
	which is further equivalent to \eqref{eq:LOSSO3} by variable substitution $\Sigma=\hat{V}^TX\hat{U}$. Furthermore, it is readily known that $X=\hat{V}\hat{\Sigma}\hat{U}^T$ is an optimal solution to the LOSSO problem. This completes the proof.
\end{proof}

According to the above results and analysis, we only need to determine $\hat{\Sigma}=X(u^*,\theta^*)$ by solving
\begin{equation}\label{eq:LOSSO5}
\begin{split}
& (u^*,\theta^*)=\arg\min_{u, \theta} ~tr(\hat{D}X(u,\theta))\\
& ~~s.t.~u^2_1+u^2_2+u^2_3=1, -\pi\leq \theta\leq \pi
\end{split}
\end{equation}
Since $\hat{D}$ is a diagonal matrix,  problem \eqref{eq:LOSSO5}  can be simplified as 
\begin{align}
	\begin{split}\label{rotation matrix}
		& (u^*,\theta^*)=\min_{u,\theta} \sum_{i=1}^3\sigma_i(\cos{\theta} +(1-\cos{\theta}) u^2_i) \\
		&\qquad s.t. ~u^2_1+u^2_2+u^2_3=1, -\pi\leq \theta\leq \pi.
	\end{split}
\end{align}
where $\sigma_i$ denotes the $i$-th diagonal element of  $\hat{D}$.

The optimal solution to problem \eqref{rotation matrix} can be derived in closed-form as follows. According to the fact of SVD, all the diagonal elements of $D$ are nonnegative. As a result, we have either \emph{case 1}: $\sigma_i\leq 0$, $\forall i$, or \emph{case 2}: $\sigma_i\geq 0$, $\forall i$. On the other hand,  it is noted that the term $\cos{\theta} +(1-\cos{\theta}) u^2_i$ must lie in between $-1$ and $1$. Therefore, for case 1 we can easily obtain the optimal solution $\theta^*=0$ and $u^*$ is arbitrary, yielding the minimum value $\sum_{i=1}^3\sigma_i$. While for case 2, we must have $\theta^*\neq 0$. With fixed $\theta^*\neq 0$, solving \eqref{rotation matrix} with respect to $u$ yields $u_{i_m}^{*2}=1$ and $u_i^*= 0$, $i\neq i_m$, where $i_m\triangleq\arg\min_{i=1,2,3} \sigma_i$. In this case, it is readily seen that minimizing the objective of \eqref{rotation matrix} is equivalent to minimizing $\sigma_{i_m}+\cos\theta\sum_{i\neq i_m} \sigma_i$, whose minimum is achieved when $cos\theta^* = -1$. In conclusion, for case 2 we can obtain the optimal solution to problem \eqref{rotation matrix} as $\theta^*=\pi$ and $u_{i_m}^*=1$, $u_i^*= 0$, $i\neq i_m$. Therefore, for both cases we can write $\hat{\Sigma}$ in a unified form as follows 
$\hat{\Sigma} = \left[
\begin{matrix}
2u_1^*-1 & 0 & 0 \\
0 & 2u_2^*-1 & 0 \\
0 & 0 & 2u_3^*-1
\end{matrix}
\right].$


According to the above results, we have the following key theorem.
\begin{theorem}
	Let $\hat{U}$, $\hat{D}$ and $\hat{V}$ be defined as in Lemma \ref{lemma2}, and $\sigma_i$ is the $i$-th element of $\hat{D}$ with $i_m\triangleq\arg\min_{i=1,2,3} \sigma_i$. Then $X=\hat{V}\hat{\Sigma}\hat{U}^T$ is an optimal solution to the LOSSO problem where 
	$$\hat{\Sigma} = \left[
	\begin{matrix}
	2u_1-1 & 0 & 0 \\
	0 & 2u_2-1 & 0 \\
	0 & 0 & 2u_3-1
	\end{matrix}
	\right]$$, with $u_i=1$, $\forall i$, if $\sigma_i\leq 0$, $\forall i$, and $u_{i_m}=1$, $u_i=0$, $\forall i\neq i_m$ if $\sigma_i \geq  0$, $\forall i$. Moreover, $\hat{D} = D$, if $\sigma_i \geq  0$, $\forall i$.
\end{theorem}

Based on Theorem 1, we give the procedure for solving the LOSSO problem in Algorithm 1, which is frequently called in both BCD-based and SUM-based rotation averaging algorithms. Furthermore, we present the BCD-based rotation averaging algorithm in Algorithm 2. This algorithm has guaranteed convergence to stationary points of the rotation averaging problem \cite{bertsekas1997nonlinear}.
\begin{algorithm}[htp]  
	\caption{Decomposition-Based LOSSO Solution}  
	\label{algorithm1}  
	\begin{algorithmic}
		\Require  
		$A$
		\Ensure  
		$X$
		\State $U,D,V$ $\gets$ the result of SVD($A$);
		\If {$det(U) \times det(V)=1$}  
		\State $\hat{V} := V$;
		\State $\sigma_i$ $\gets$ the $i$-th element of $D$;
		\State $i_m := \arg\min_{i=1,2,3} \sigma_i$;
		\State $u_{i_m}=1$ and $u_i=0$, $\forall i\neq i_m$;
		\Else
		\State $\hat{V} := -V$;
		\State $u_i=1$, $\forall i$;
		\EndIf
		\State $\Sigma := \left[
		\begin{matrix}
		2u_1-1 & 0 & 0 \\
		0 & 2u_2-1 & 0 \\
		0 & 0 & 2u_3-1
		\end{matrix}
		\right]$;
		\State $X := \hat{V}\Sigma U^T$;
	\end{algorithmic}  
\end{algorithm}

\begin{algorithm}[htb]  
	\caption{BCD-Based Rotation Averaging Algorithm}  
	\label{algorithm2}  
	\begin{algorithmic}  
		\Require  
		$R_{ij},(i,j)\in E$
		\Ensure  
		$R^{(t)}_k, k=1,2,\ldots,n$
		\State Initialize  $\epsilon$, $R^{(0)}_k, \forall k$ and set $t := 0$; 
		\Do
		\State $t := t+1$;
		\For{$i \gets 1$ to $n$}
		\State compute $A^{(t)}_i$ via \eqref{A} based on $R^{(t-1)}_k, \forall k$;
		\State compute $R^{(t)}_i$ using Alg. \ref{algorithm1} with $A=A^{(t)}_i$;
		\EndFor
		\doWhile{$\sum_{k=1,2,\ldots,n} \frac{||R^{(t)}_k - R^{(t-1)}_k||}{||R^{(t)}_k||}  \ge \epsilon$}
		
	\end{algorithmic}  
\end{algorithm}

\subsection{Successive Upper-bound Minimization}
Clearly, the BCD-based rotation averaging algorithm runs in serial. However, for large-scale rotation averaging problems (i.e, $n$ is large), parallel algorithms are more desirable. Hence, by further exploring the structure of the rotation averaging problem, we propose a parallel rotation averaging algorithm.

Our parallel algorithm is based on SUM method \cite{razaviyayn2013unified}. For the rotation averaging problem with smooth objective, SUM works in the same way as the classical Majorization-Minimization method. At each iteration, the SUM updates the variable by successively minimizing either locally tight upper-bound or strictly convex local approximations of the objective function (also called surrogate function) \cite{razaviyayn2013unified}. Under mild conditions, the SUM algorithm is guaranteed to achieve convergence towards stationary solutions.


To proceed, let us define
\begin{equation}  
R = \left[R_1\quad R_2\quad ... \quad R_n \right].
\end{equation}
Moreover, recall the definition of $\tilde{R}$ in \eqref{Rtilde}. Then we can  recast problem \eqref{Inital problem} as follows
\begin{align}
	\begin{split}\label{RRR}
		&\min_{R} \quad -tr\left(R\tilde{R}R^T\right) \\
		&s.t. \quad\quad R_i\in SO(3), i=1,2,\ldots, n.
	\end{split}&
\end{align}
The key to SUM is finding a locally tight upper bound for the objective of problem \eqref{RRR}. To this end, we find a locally tight upper bound in the following lemma.
\begin{lemma} \label{UBlemma}
	Let $B \triangleq \mu I + \tilde{R}\succeq 0$ with some appropriate $\mu$. We have for any $ R, \Bar{R}\in   SO(3)^n$
	\begin{align}\label{upperbound}
		-tr\left(R\tilde{R}R^T\right)&=\mu tr(RR^T)-tr(RBR^T)\\
		\leq 3\mu n&-2tr\left(B\Bar{R}^TR\right) + tr\left(\Bar{R}B\Bar{R}^T\right)\nonumber
	\end{align}
\end{lemma}
\begin{proof}
	First, the equality in \eqref{upperbound} holds true obviously. Second, since the matrix $B$ is positive semidefinite, the function $tr(RBR^T)$ is convex in $R$. As a result, the inequality in \eqref{upperbound} follows directly from the facts that 1) $tr(RR^T)=\sum_{i=1}^n tr(R_iR_i^T)=3n$ and 2) the convex function $tr(RBR^T)$ is lower bounded at $\Bar{R}$ by its linear approximation $2tr\left(B\Bar{R}^TR\right) - tr\left(\Bar{R}B\Bar{R}^T\right)$. This completes the proof.
\end{proof}

Lemma \ref{UBlemma} shows that a quadratic function of $R$ can be upper bounded by a linear function of $R$ on the space defined by $SO(3)^n$. Using such a locally tight upper bound, we solve in each step of the SUM method the following problem
\begin{align}
	\begin{split}\label{BRR}
		&\min_{R} \quad Tr\left(A_{\bar{R}}^TR\right)\\
		&s.t. \quad\quad R\in SO(3)^n
	\end{split}&
\end{align}
where $A_{\bar{R}} = - \Bar{R}B^T$.
Note that the objective of the above problem is linear in $R$ and thus is separable across $R_i$, $i=1,2,\ldots,n$. Therefore, by writing $A_{\bar{R}}$ as
\begin{equation}
A_{\bar{R}}=\left[A_1\quad A_2\quad ... \quad A_n \right].
\end{equation}
we can decompose problem \eqref{BRR} into $n$ independent LOSSO problems ($i=1,2,\ldots,n$)
\begin{align}
	\begin{split}\label{CR}
		&\min_{R_i} \, tr\left(A^T_iR_i\right)\\
		&s.t. \quad R^T_iR_i = I, det(R_i)=1.
	\end{split}&
\end{align}
which can be easily solved using Algorithm 1.

The SUM-based rotation averaging algorithm is summarized in Algorithm 3. Similar to \cite{razaviyayn2013unified}, it can be shown that the SUM-based rotation averaging algorithm has guaranteed convergence to stationary solutions. 


\begin{algorithm}[htb]  
	\caption{SUM-Based Rotation Averaging Algorithm}
	\label{algorithm3}  
	\begin{algorithmic}  
		\Require  
		$\tilde{R}$
		\Ensure  
		$R^{(t)}$
		\State Initialize $\epsilon$, $R^{(0)}$ and set $t := 0$;
		\State $\lambda_{\min}$ $\gets$ the minimum eigenvalue of $\tilde{R}$;
		\State $\mu := |\min(\lambda_{\min}, 0)|$; 
		\State $B \gets \mu I +\tilde{R}$;
		\Do
		\State $t := t+1$;
		\State $A^{(t)}_{\bar{R}} := -R^{(t-1)}B$;
		\For{$i \gets 1$ to $n$}
		\State $A^{(t)}_i$ $\gets$ the $i$-th column block of $A^{(t)}_{\bar{R}}$;
		\State compute $R^{(t)}_i$ using Alg. \ref{algorithm1} with $A={A^{(t)}_i}^{T}$;
		\EndFor
		\doWhile{$\frac{||R^{(t)} - R^{(t-1)}||}{{||R^{(t)}||}} \ge \epsilon$}
		
	\end{algorithmic}  
\end{algorithm}

\begin{table*}[htb]
	\scriptsize
	\centering
	\caption{Comparison of resulting average errors, minimum eigenvalue and running times on synthetic data. 
		Here avg.error is the average objective value of \eqref{eq:fnorm} divided by $|E|$, $min(\lambda_{min})$ is the minimum of the smallest eigenvalues of $\Lambda-\tilde{R}$ (defined in Theorem 1) over 100 runs, which examines whether the solution satisfies the optimal condition.}
	\begin{tabular}{cccccccccccc}
		\toprule
		\multirow{1}{*}{} & \multirow{1}{*}{} & \multirow{1}{*}{} & \multicolumn{3}{c}{BCDSR \cite{eriksson2018rotation}} & \multicolumn{3}{c}{Alg. \ref{algorithm2} (BCD-Based)} & \multicolumn{3}{c}{Alg. \ref{algorithm3} (SUM-Based)} \\
		\cmidrule(r){4-6} \cmidrule(r){7-9} \cmidrule(r){10-12}
		$n$ & $\varphi[rad]$ & $p$
		&  avg.error      &  $min(\lambda_{min})$      &  time[s]
		&  avg.error      &  $min(\lambda_{min})$      &  time[s]
		&  avg.error      &  $min(\lambda_{min})$      &  time[s]  \\
		\midrule
		20            &0.2     &0     & 0.2102            & -3.46e-15             & 0.2390          & 0.2098            & -1.08e-16          & 0.0120  
		& 0.2098            & 7.06e-16           & 0.0069
		\\
		&     &0.3     & 0.2019            & -2.40e-14         
		& 0.2352          & 0.2018            & -9.26e-17          & 0.0124  
		& 0.2018            & -1.69e-16           & 0.0070
		\\
		&0.5   &0       & 0.5026            & -9.06e-15       &0.2438
		& 0.5013           & -6.88e-15          & 0.0126
		& 0.5013           & -1.43e-15          & 0.0095
		\\
		&     &0.3     & 0.5088            & 5.19e-15         
		& 0.2671          & 0.5086            & -2.79e-15          & 0.0124  
		& 0.5086            & -1.11e-14           & 0.0105
		\\
		200           &0.2     &0     & 0.2230            & -1.05e-11    & 59.22           & 0.2229          & -1.19e-14          & 0.1126  
		& 0.2229          & 2.12e-14           & 0.0622
		\\
		&     &0.3     & 0.2176            & -1.41e-13    & 55.82
		& 0.2163            & -2.90e-14          & 0.1245  
		& 0.2163            & 3.36e-15           & 0.0734
		\\
		&0.5   &0       & 0.5505            & -8.38e-12     & 55.72 
		& 0.5502         & -2.11e-14          & 0.1236
		& 0.5502         & 3.99e-16           & 0.0635
		\\
		&     &0.3     & 0.5513            & -2.04e-14    & 57.54          & 0.5539            & -4.17e-14          & 0.1482 
		& 0.5539            & -1.10e-15           & 0.0640
		\\
		2000          &0.2    &0      & 0.2409            & -127.11     & 15139.34
		& 0.2248         & 7.78e-14          & 9.53  
		& 0.2248         & -1.68e-13          & 2.54
		\\
		&     &0.3     & 0.2398            & -134.41    & 15278.67
		& 0.2232            & -2.49e-13          & 9.52  
		& 0.2232            & -3.00e-13           & 2.55
		\\
		&0.5    &0      & 0.5699            & -153.72    & 15374.61 
		& 0.5526         & -3.64e-13          & 9.64
		& 0.5526         & -2.74e-13          & 2.59
		\\
		&     &0.3     & 0.5783            & -156.04    & 15392.96 
		& 0.5569            & 3.49e-13          & 9.97 
		& 0.5569            & -6.42e-13           & 2.61
		\\
		\bottomrule
	\end{tabular}
	\label{tab:1}
\end{table*}

\begin{figure*}[htbp]
	\setlength{\abovecaptionskip}{0.cm}
	\setlength{\belowcaptionskip}{-0.5cm}   
	\centering 
	\subfigure[BCD-Based]{ 
		\includegraphics[height=6.3cm]{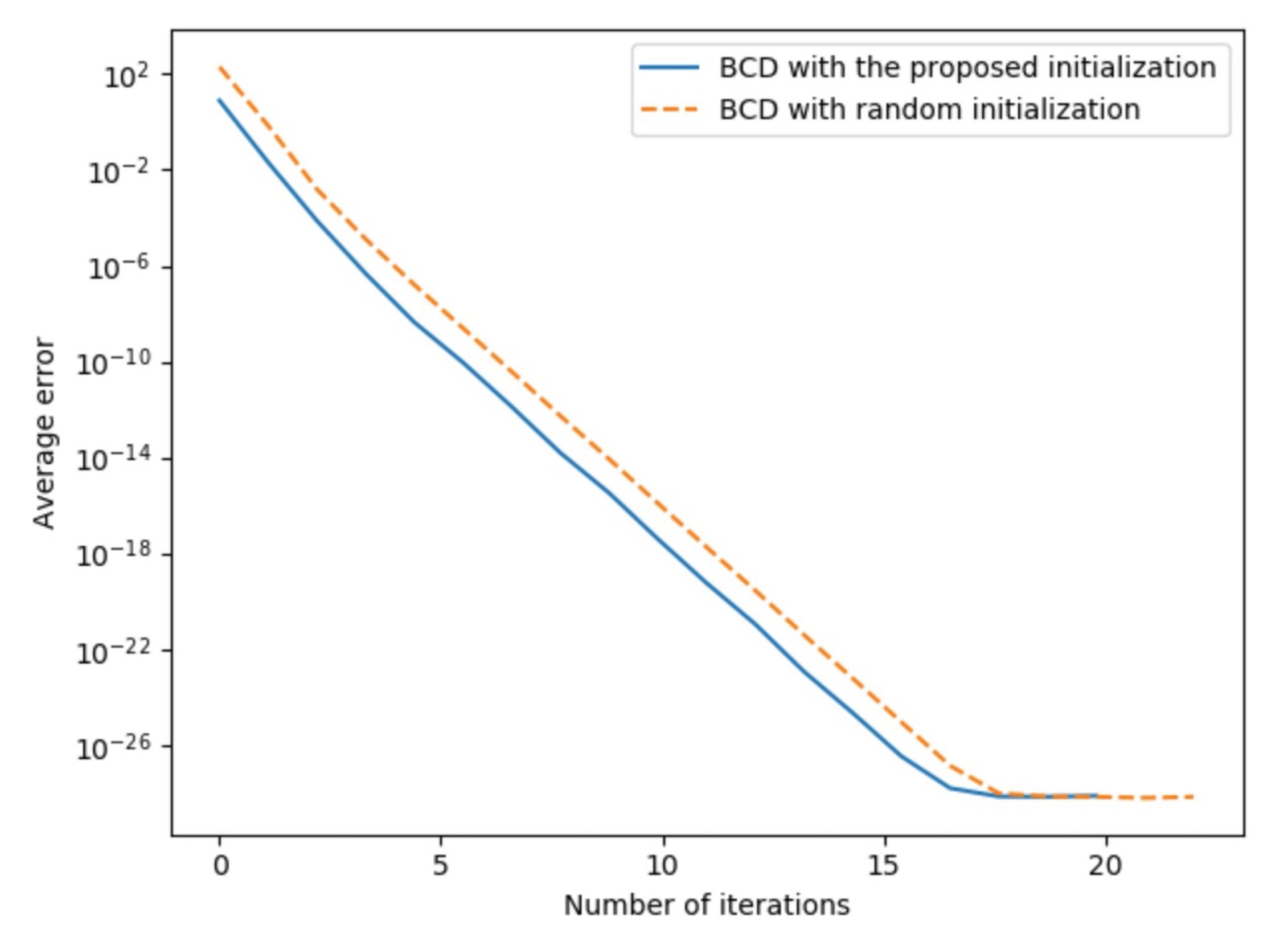} 
	} 
	\subfigure[SUM-Based]{ 
		\includegraphics[height=6.3cm]{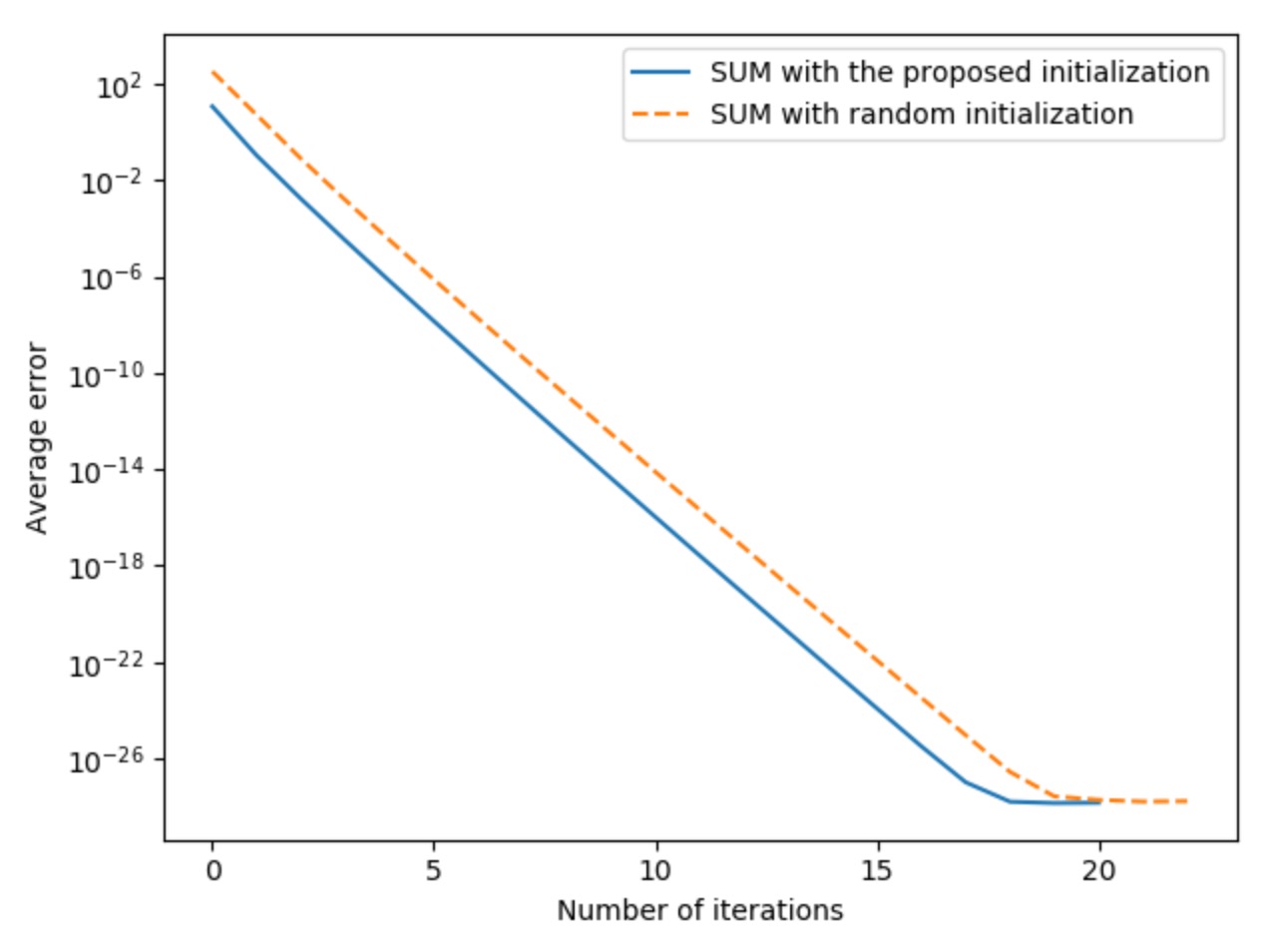} 
	} 
	\caption{Convergence examples of BCD-based and SUM-based rotation averaging algorithms with different initializations.}
	\label{3}
\end{figure*}

\subsection{{Globally Optimal Solution to the Noiseless Case}}
The above two algorithms require feasible initialization. In general, a good initialization is desirable because it could speed up the convergence. To this end, we propose a closed-form optimum solution to the noiseless case below, which can serve as a good feasible initialization for the proposed two algorithms  in the noisy case. We emphasize that to the best of our knowledge even in the noiseless case the rotation averaging problem \eqref{eq:fnorm} has not yet been globally solved in the literature.

Let us describe the \textbf{R}otation \textbf{A}veraging problem as an $n$-vertex connected graph $G = \left\{V, E\right\}$, named RA graph, where each vertex in $V$ represents an absolute rotation and each edge $(i,j)\in E$ corresponds to a relative rotation $R_{ij}$ if it exists. For a connected graph, there must exist a \emph{shortest path} linking each pair of vertices. Resorting to the concept of shortest path, we formally present in Lemma \ref{lemma7} a closed-form optimum solution to the noiseless case, which is a good feasible point for the two proposed rotation averaging algorithms above.

\begin{lemma} \label{lemma7} Let $j_i$ denote the $j$-th vertex on the shortest path of the RA graph from vertex $1$ to vertex $i$ with identification $1_i=1$ and $(m+1)_i=i$. Then $\hat{R}_1=I$, and $\hat{R}_i = R_{12_i}R_{2_i3_i}R_{3_i4_i}...R_{m_ii}$ is an optimum solution to the rotation averaging problem \eqref{eq:fnorm} in the noiseless case.
\end{lemma}

See Appendix B for the detailed proof and toy examples.

\section{Experiments}
In this section, we empirically evaluate the accuracy and the computational efficiency of the proposed algorithms on a computer with 32 GB Memory and 2.4 GHz processor. We first introduce our experiment configuration on synthetic data. Next, we compare the proposed algorithms against the state-of-the-art method BCDSR \cite{eriksson2018rotation}. Then, the effects of the proposed special initialization are investigated. Finally, we show the performance of our SUM-Based algorithm on real-world data.
Note that all algorithms here are implemented in Python. 

\subsection{Synthetic Data}
We use complete graphs to construct synthetic data. In this part, we randomly generate a set of absolute rotation $R_i$ to form a complete graph. Each relative rotation is computed in the presence of noise. The noise are the random rotations obtained by an axis uniformly randomly sampled from the unit sphere and an angle generated from normal distribution with mean $0$ and variance $\varphi$. For addition, the parameter $p$ is used to control the sparsity of the generated graphs, meaning that only p percentage of relative rotations are used in rotation averaging. In the simulations, we set $\varphi$ to be $0.2$ radian or $0.5$ radian and $p$ to be 0 or 0.3 in our four sets of experiments as shown in Table \ref{tab:1}.

\subsubsection{Accuracy}
In Table \ref{tab:1}, we compare our methods with BCDSR \cite{eriksson2018rotation}. Each reported number is averaged over 100 runs except for $min(\lambda_{min})$. We can see that the convergence time of all the algorithms increases with $n$. In addition, the average errors of Alg.  \ref{algorithm2} and Alg. \ref{algorithm3} are slightly lower than that of BCDSR. Further, from this table, we also observe that we almost have $min(\lambda_{min}) \ge 0$ (the meaning of this term is explained in the caption of Table 1) for both Alg. \ref{algorithm2} and Alg. \ref{algorithm3} (note here that the extremely small negative numbers should be numerical errors), which imply that globally optima are reached in terms of the result of Theorem 1. In contrast, we have $min(\lambda_{min}) < 0$ for BCDSR algorithm when $n = 2000$, which is possibly due to either local convergence or that the algorithm terminates before the real convergence. Therefore, the proposed BCD-based and SUM-based algorithms yield better accuracy than the BCDSR.

\begin{figure}[htbp]
	\centering
	\includegraphics[height=6.2cm]{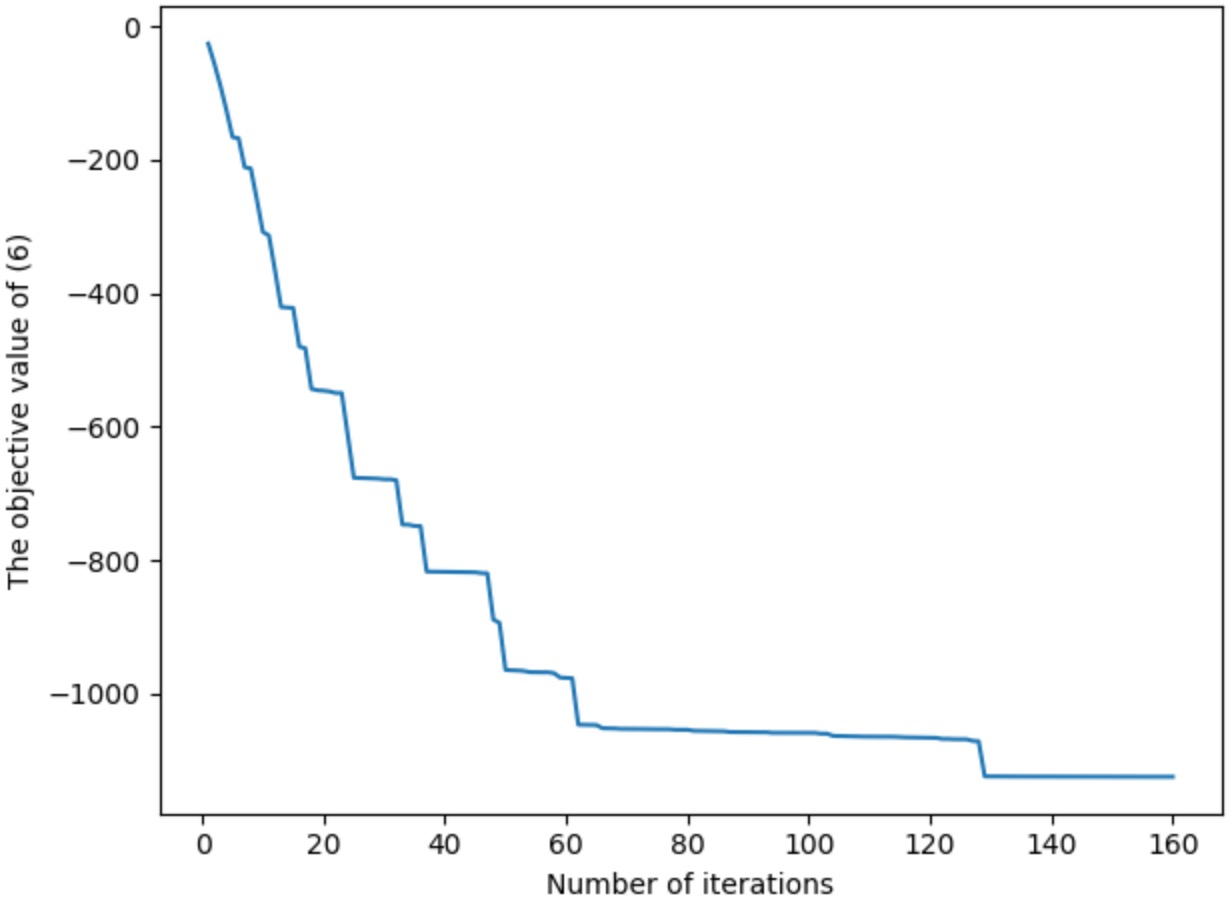}
	\caption{A convergence example of the BCDSR algorithm.}
	\label{4}
\end{figure}

\subsubsection{Computational Efficiency}
Figure \ref{3} illustrates the effect on convergence of using the proposed special initialization. It can be seen from Fig. \ref{3} that the proposed initialization can effectively speed up the convergence. Also, as shown in Fig. \ref{4}, BCDSR not only suffers from its slow convergence, the evolution of objective values of \eqref{RY} often plateaus during iterations, this makes it difficult to set convergence criterion for the algorithm. Moreover, although Alg. \ref{algorithm2} and Alg. \ref{algorithm3} have very similar convergence behaviors, they spend different time in each iteration. Thus, while it can be shown that the three algorithms have the same order computational complexity (i.e., $O(n^2)$), Algorithm \ref{algorithm3} performs  best among them in terms of cpu time.

\begin{table}[htb]
	\scriptsize
	\centering
	\caption{Comparison of resulting average errors and running times on real-world data. Here, avg.error and times are the same as those in Table \ref{tab:1}.}
	\begin{tabular}{cccccc}
		\toprule
		\multirow{1}{*}{} & \multirow{1}{*}{} &\multicolumn{2}{c}{BCDSR \cite{eriksson2018rotation}} &  \multicolumn{2}{c}{Alg. \ref{algorithm3} (SUM-Based)}\\
		\cmidrule(r){3-4} \cmidrule(r){5-6}
		Dataset  & $n$ 
		&  avg.error  &  time[s]
		&  avg.error &  time[s]\\
		\midrule
		ET\cite{snavely2006photo}            &9        & 1.40            & 0.298             & 0.760          & 0.021 \\
		Kermit\cite{snavely2006photo}        &11         & 1.23            & 0.681 
		
		& 0.958          & 0.027            
		\\
		Tsinghua Gate\cite{web}        &68         & 1.64            & 6.592 
		
		&1.64              & 0.483           
		\\
		Zhantan Temple\cite{web}       &158        & 1.67            & 34.148         
		& 1.60          & 0.596
		\\
		Fayu Temple\cite{web}       &290        & 1.65            & 104.49     
		
		& 1.56          & 1.96
		\\
		\bottomrule
	\end{tabular}
	\label{tab:2}
\end{table}

\subsection{Real-World Data}
In this set of experiments, the image collections are all obtained from public datasets \cite{snavely2006photo,web}. We use images to generate the relative rotations $R_{ij}$ via SIFT \cite{ng2003sift}. Then we run the proposed SUM algorithm and the BCDSR algorithm to recover the absolute rotations $R_i's$. Table 2 lists the computational results and the average time required by the two algorithms. It is observed that the proposed SUM algorithm outperforms  the BCDSR in both computational time and error. Particularly, as the problem size $n$ increases, the computational time gap between the two algorithms becomes widened. For example, when $n=290$, the computational time of the BCDSR is nearly 53 times that of the SUM algorithm. In addition, we use an open-source code bundler \cite{snavely2006photo} and the SUM-based algorithm to run some reconstruction examples. For instance, the reconstruction results of Tsinghua Gate and Zhantan Temple are presented in Fig. 4 (a) and Fig. 4 (b) respectively. It can be seen that the reconstruction results are quite good.

\begin{figure}[htbp]
	\setlength{\abovecaptionskip}{0.cm}
	\setlength{\belowcaptionskip}{-0.5cm}   
	\centering 
	\includegraphics[height=3.5cm]{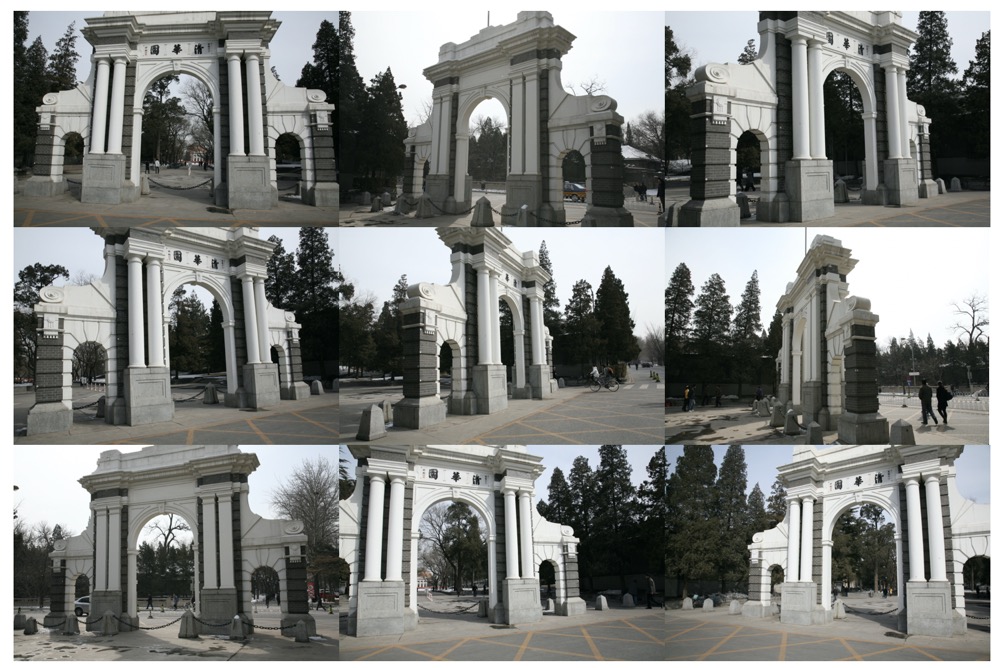}
	\includegraphics[height=3.5cm]{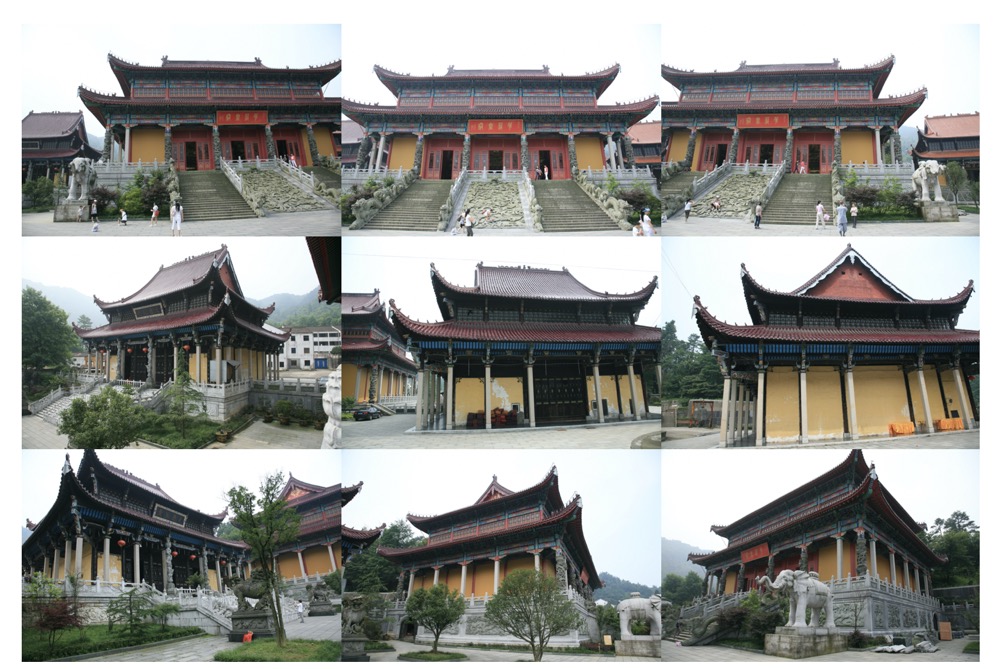} \\
	\includegraphics[height=3.6cm]{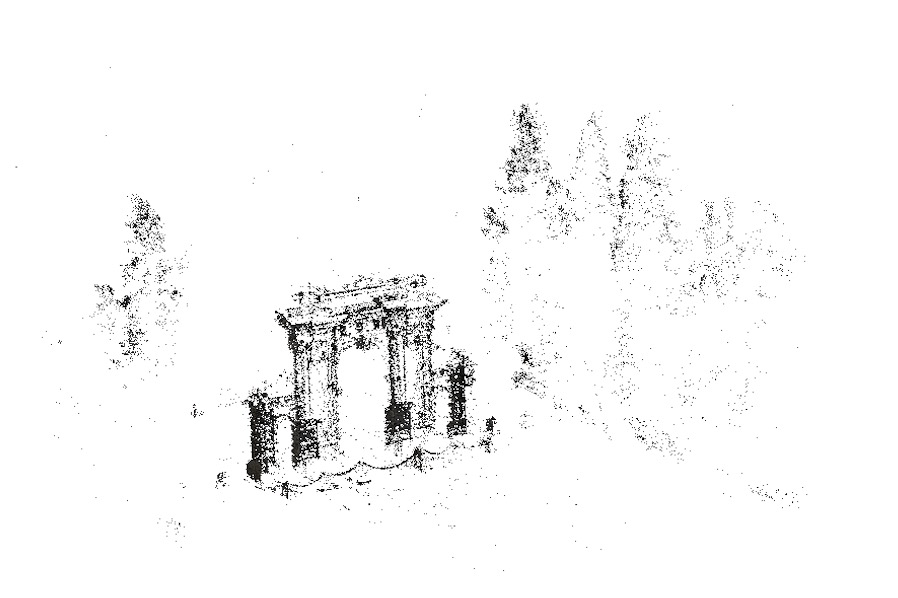}
	\includegraphics[height=3.6cm]{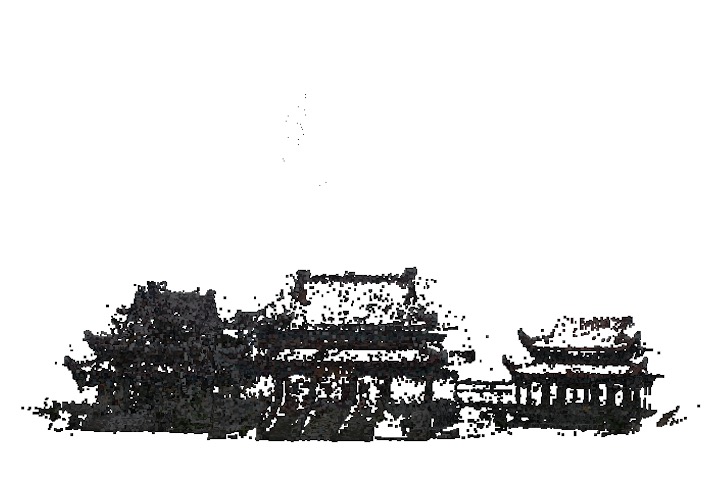}\\
	\quad (a) Tsinghua Gate \qquad (b) Zhantan Temple
	\caption{Images and reconstructions.}
	\label{5}
\end{figure}

\section{Conclusions}
In this paper, we have proposed a BCD-based algorithm and a SUM-based algorithm for the rotation averaging problem. Both algorithms significantly outperform the baseline algorithm. While these algorithms are guaranteed to converge to stationary points, our numerical results show that they actually reach global optima. In addition, the SUM-based algorithm can further benefit from parallel computation. Hence, the proposed algorithms significantly outperform the baseline algorithm.

\begin{acks}
	This work was supported by the National Key Research and Development Project under grant 2017YFE0119300.
\end{acks}

\bibliographystyle{ACM-Reference-Format}
\bibliography{acm}

\appendix
\section{Proof of Problem 11}
\setcounter{equation}{8}
\begin{align}
	\begin{split}\label{sub-problem}
		&\min_{R_l \in \mathbb{R}^{3\times3}} \quad -\sum_{(i,j)\in E} tr \left( R_iR_{ij}R^T_j  \right)\\
		&s.t.\quad\quad R^T_lR_l = I, det(R_l)=1.
	\end{split} 
\end{align} 

\begin{lemma}
	Assume that $(p,q)\in E$ and $p,q$ $ \neq l$ , define
	\begin{equation}
	A_l \triangleq -\sum_{(l,q)\in E} R_{lq}R^T_q  -\sum_{(p,l) \in E} R^T_{pl}R^T_p \label{definition}
	\end{equation}
	Afterwards, the subproblem \eqref{sub-problem} with respect to $R_l$ can be simplified to 
	\begin{align}
		\begin{split}
			&\min_{R_l \in \mathbb{R}^{3\times3}} \quad tr\left(A_lR_l\right) \\
			&s.t. \quad R^T_lR_l = I, det(R_l)=1.\\
		\end{split}& 
	\end{align} 
\end{lemma}
\begin{proof}
	Without loss of generality, let us first extract the $R_l$-related terms from the objective function \eqref{sub-problem}. Obviously, we can split the objective function into the following three parts
	\begin{align*}
		-\sum_{(i,j)\in E} tr \left( R_iR_{ij}R^T_j  \right)= -\sum_{(l,q)\in E} tr \left( R_lR_{lq}R^T_q  \right)\\ -\sum_{(p,l) \in E} tr \left( R_pR_{pl}R^T_l  \right) -\sum_{(p,q)\in E} tr \left( R_pR_{pq}R^T_q  \right).
	\end{align*}
	where the last part can be viewed as a constant because it is independent of $R_l$. Thus, the objective function can be simply expressed as follows
	\begin{equation*}
		-\sum_{(l,q)\in E} tr \left( R_lR_{lq}R^T_q  \right) -\sum_{(p,l) \in E} tr \left(R_pR_{pl}R^T_l \right)+C 
	\end{equation*}
	where $C$ is a constant. Further, due to the identity $trace(M) = trace\left(M^T\right)$, we have $trace\left(R_pR_{pl}R^T_l\right)=trace\left( \left( R_pR_{pl}R^T_l  \right)^T \right)$. Hence, we further express the objective function as
	\begin{align*}
		&-\sum_{(l,q)\in E} tr \left( R_lR_{lq}R^T_q  \right) -\sum_{(p,l) \in E} tr \left( \left( R_pR_{pl}R^T_l  \right)^T \right)+C\\
		&=-\sum_{(l,q)\in E} tr \left( R_lR_{lq}R^T_q  \right) -\sum_{(p,l) \in E} tr \left( R_lR^T_{pl}R^T_l\right)+C\\
		&=tr \left(R_l \left(-\sum_{(l,q)\in E} R_{lq}R^T_q -\sum_{(p,l) \in E} R^T_{pl}R^T_p \right) \right)+C
	\end{align*}
	By using the definition \eqref{definition}, the formula above can be expressed as follows
	\begin{equation*}
		tr(R_lA_l)+C  
	\end{equation*}
	Finally, since $tr\left(XY\right)=tr\left(YX\right)$, rewrite it as 
	\begin{equation*}
		tr(A_lR_l)+C  
	\end{equation*}
	$C$ is unrelated variable with regard to $R_l$. Hence, this completes the proof.
\end{proof}

\section{Proof of Lemma 3}
\setcounter{equation}{2}
\begin{equation} \label{RRRF}
\mathop {\min }\limits_{R_1,...,R_n\in SO(3)} \sum_{(i,j)\in E} \big|\big|R_iR_{ij}-R_j\big|\big|^2_F.
\end{equation}
\begin{lemma} Let $j_i$ denote the $j$-th vertex on the shortest path of the RA graph from vertex $1$ to vertex $i$ with identification $1_i=1$ and $(m+1)_i=i$. Then $\hat{R}_1=I$, and $\hat{R}_i = R_{12_i}R_{2_i3_i}R_{3_i4_i}...R_{m_ii}$ is an optimum solution to the rotation averaging problem \eqref{RRRF} in the noiseless case.
\end{lemma}
\begin{proof}
	Suppose that $\dot{R}_i, i=1,...,n$ is an optimum solution to the rotation averaging problem \eqref{RRRF} in the noiseless case. That is, it satisfies $\dot{R}_iR_{ij}=\dot{R}_j$, $\forall (i,j)\in E$, equivalently, $R_{ij}=\dot{R}_i^T\dot{R}_j$, $\forall (i,j)\in E$. Under this condition, it suffices to show that $\hat{R}_i$, $i=1,2,\ldots, n$, satisfies the $SO(3)^n$ constraint and $\hat{R}_iR_{ij}=\hat{R}_j$, $\forall (i,j)\in E$.
	
	First, it is not difficult to show that
	$$\hat{R}^T_1\hat{R}_1=I, det(\hat{R}_1)=1.$$
	and 
	\begin{align*}
		\begin{split}
			&\hat{R}^T_i\hat{R}_i = \left(R^T_{m_ii}...R^T_{12_i}\right)R_{12_i}...R_{m_ii} = I, \forall i\\
			&det(\hat{R}_i) = det\left(R_{12_i}...R_{m_ii}\right)\\&= det(R_{12_i})\times ...\times det(R_{m_ii})
			\\& = 1,\forall i
		\end{split}
	\end{align*}
	Second, by the definition, we have 
	\begin{align*}
		\begin{split}
			&\hat{R}_iR_{ij}=\left(R_{12_i}R_{2_i3_i}...R_{m_ii}\right)R_{ij}\\
			&=\dot{R}_1^T\dot{R}_j=R_{12_j}R_{2_j3_j}...R_{\ell_jj}=\hat{R}_j
		\end{split}
	\end{align*}
	where we have used the fact $R_{ij}=\dot{R}_i^T\dot{R}_j$, $\forall (i,j)\in E$ in the second and third equalities.
	This completes the proof.
\end{proof}

For better understanding Lemma 3, let us give two toy examples in Figure \ref{2}. In a complete graph, for each $i\neq 1$, there is an edge linking vertex $1$ and vertex $i$. In this case, $R_1=I$ and $R_i=R_{1i}$, $\forall i\neq 1$ constitute a globally optimal solution to the rotation averaging problem. In a cyclic graph, suppose without loss of generality that vertices are linked in a sequential order, i.e., $1,2,\ldots, n$. In this case, $R_1=I$ and $R_i=R_{12}R_{23}\ldots R_{(i-1)i}$, $\forall i\neq 1$ form a globally optimal solution.

\setcounter{figure}{4}
\begin{figure}[htbp]
	\setlength{\abovecaptionskip}{0.cm}
	\setlength{\belowcaptionskip}{-0.5cm}   
	\centering 
	\subfigure[complete graph]{ 
		\includegraphics[height=5.1cm]{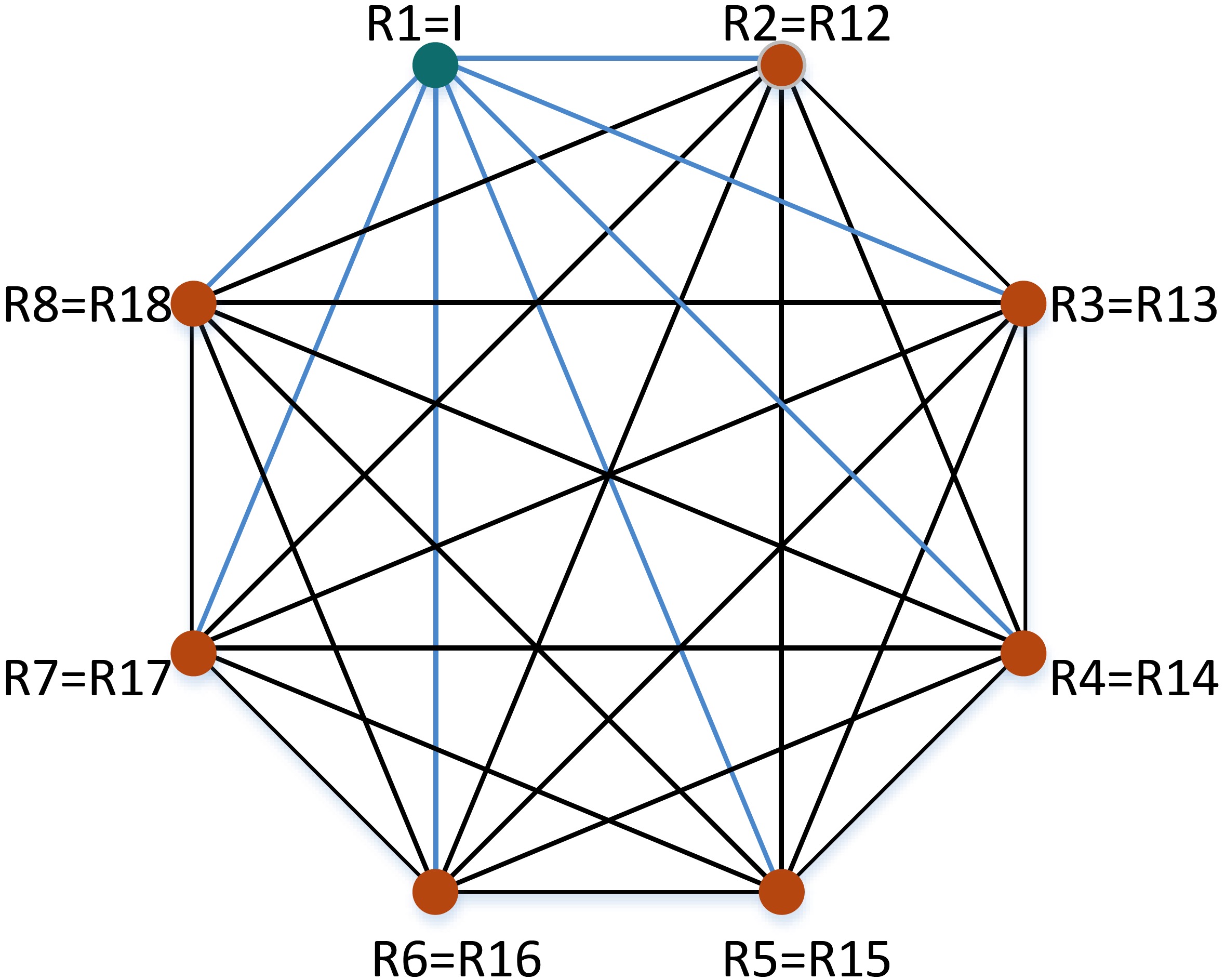} 
	} 
	\subfigure[cyclic graph]{ 
		\qquad \includegraphics[height=5.1cm]{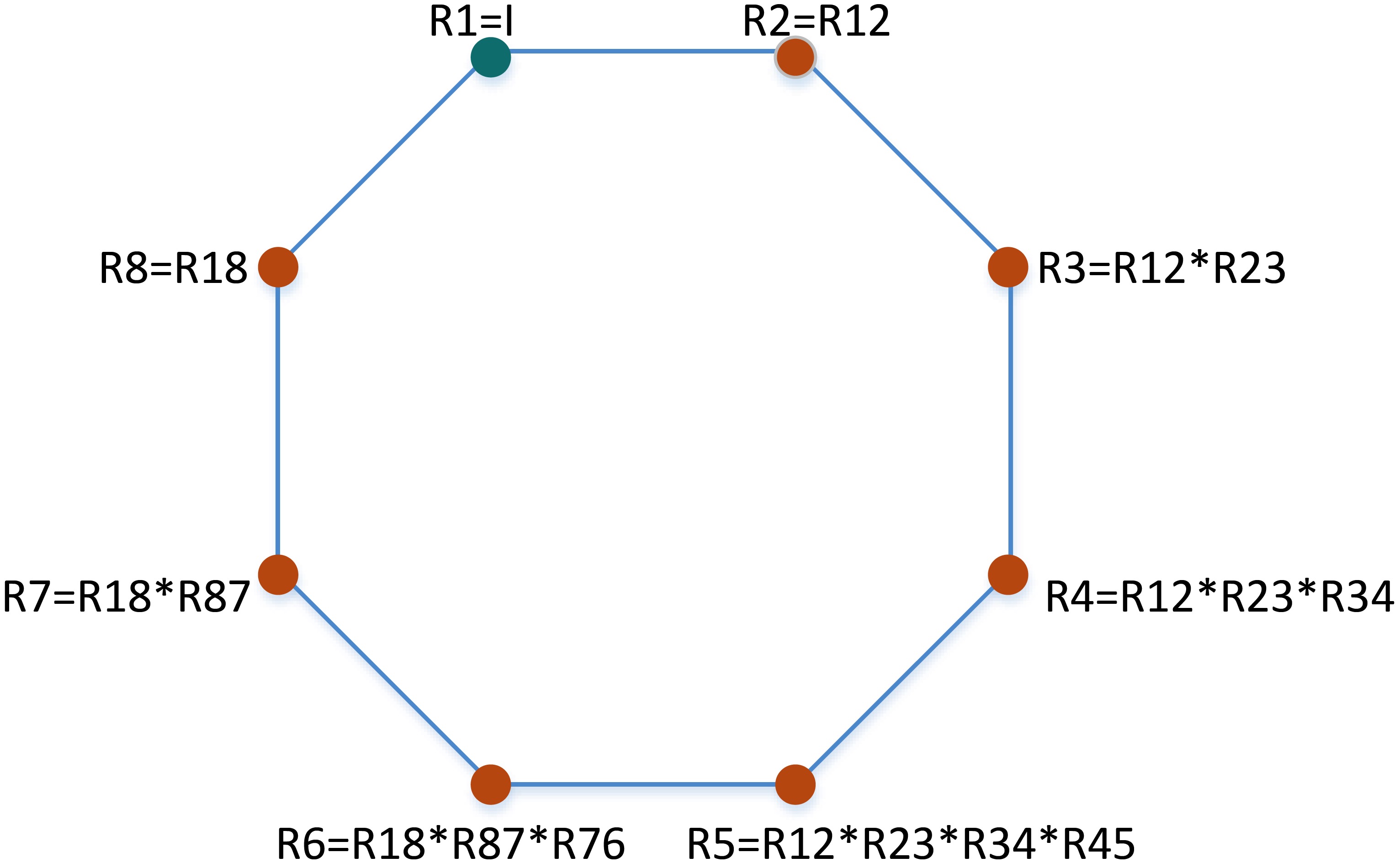} 
	} 
	\caption{The globally optimal solutions for complete graph and cyclic graph in the noiseless case.}
	\label{2}
\end{figure}

\end{document}